\documentclass[envcountsame,runningheads]{llncs}

\usepackage[utf8]{inputenc}
\usepackage{xspace}
\usepackage{amssymb}
\usepackage{amsmath}
\usepackage{amsthm}
\usepackage{xcolor}
\usepackage[defblank]{paralist}
\usepackage{enumitem}
\usepackage{thmtools}
\usepackage[strings]{underscore} 
\usepackage[linesnumberedhidden,ruled,vlined]{algorithm2e}

\spnewtheorem{claimn}{Claim}{\it}{}

\usepackage{hyperref}

\makeatletter
\newcommand{\nobrackettag}[0]{\def\tagform@##1{\maketag@@@{##1}}}
\makeatother

\newcommand{\envendmark}{\mbox{$\triangleleft$}}
\newcommand{\envend}{\hfill\envendmark}

\newcommand{\p}{\varphi}
\newcommand{\q}{\psi}

\newcommand{\ua}{\underline{a}}
\newcommand{\ud}{\underline{d}}
\newcommand{\ut}{\underline{t}}
\newcommand{\ux}{\underline{x}}
\newcommand{\uy}{\underline{y}}
\newcommand{\uz}{\underline{z}}

\newcommand{\sttr}[2][]{\ensuremath{T_{#1}(#2)}\xspace}

\newcommand{\ex}[1]{\textsf{#1}}

\newcommand{\sig}[1]{\ensuremath{\Sigma_{#1}}\xspace}

\newcommand{\Var}{\ensuremath{\mathsf{Var}}\xspace}
\newcommand{\ar}{\ensuremath{\mathsf{ar}}}
\newcommand{\QE}{\ensuremath{\mathsf{QE}}}

\newcommand{\Pre}{\ensuremath{\textit{Pre}}\xspace}

\newcommand{\assign}{\ensuremath{\mathfrak{a}}\xspace}

\newcommand{\dom}[1][\Imc]{\Delta^{#1}}
\newcommand{\Int}[2][\Imc]{#2^{#1}}

\newcommand{\ISA}{\sqsubseteq}

\newcommand{\SOMET}[1]{\exists #1.\top}
\newcommand{\SOME}[2]{\exists #1.#2}
\newcommand{\AND}{\sqcap}
\newcommand{\NOT}{\lnot}

\newcommand{\NC}{\ensuremath{\mathsf{N}_{\mathsf{C}}}\xspace}
\newcommand{\NI}{\ensuremath{\mathsf{N}_{\mathsf{I}}}\xspace}

\newcommand{\NR}{\ensuremath{\mathsf{N}_{\mathsf{R}}\xspace}}
\newcommand{\NPr}{\ensuremath{\mathsf{N}_{\mathsf{P}}\xspace}}
\newcommand{\NF}{\ensuremath{\mathsf{N}_{\mathsf{F}}\xspace}}

\newcommand{\ourDL}{\textit{RDFS}\ensuremath{\!_{_+}}\xspace}

\newcommand{\dllite}{\textit{DL-Lite}\xspace}

\newcommand{\NP}{\textsc{NP}\xspace}
\newcommand{\PSpace}{\textsc{PSpace}\xspace}
\newcommand{\NPSpace}{\textsc{NPSpace}\xspace}

\newcommand{\Amc}{\ensuremath{\mathcal{A}}\xspace}

\newcommand{\Imc}{\ensuremath{\mathcal{I}}\xspace}
\newcommand{\Jmc}{\ensuremath{\mathcal{J}}\xspace}

\newcommand{\Omc}{\ensuremath{\mathcal{O}}\xspace}

\newcommand{\Smc}{\ensuremath{\mathcal{S}}\xspace}
\newcommand{\Tmc}{\ensuremath{\mathcal{T}}\xspace}

\begin{document}

\title{SMT-Based Safety Verification of\\
 Data-Aware Processes under Ontologies\\
 (Extended Version)}
\titlerunning{SMT-Based Verification of DAPs under
 Ontologies (Ext. Vers.)}

\author{Diego Calvanese\inst{1,2} \and
 Alessandro Gianola\inst{1} \and\\
 Andrea Mazzullo\inst{1} \and
 Marco Montali\inst{1}}
\authorrunning{D.~Calvanese et al.}

\institute{KRDB Research Centre for Knowledge and Data\\
 Free University of Bozen-Bolzano, Italy\\
 \email{\textit{surname}@inf.unibz.it}
 \and Computing Science Department, Ume\aa\ University, Sweden}

\maketitle




\begin{abstract}
  In the context of verification of data-aware processes (DAPs), a formal approach
  based on \emph{satisfiability modulo theories} (SMT) has been considered to
  verify parameterised safety properties of so-called \emph{artifact-centric
   systems}.  This approach requires a combination of model-theoretic notions
  and algorithmic techniques based on \emph{backward reachability}.
  We introduce here a variant of one of the most investigated models in this
  spectrum, namely \emph{simple artifact systems} (\emph{SASs}), where, instead
  of managing a database, we operate over a description logic (DL) ontology
  expressed in (a slight extension of) RDFS.
  This DL, enjoying suitable model-theoretic properties, allows us to define
  DL-based SASs to which backward reachability can still be applied, leading to
  decidability in \PSpace of the corresponding safety problems.
\end{abstract}

\section{Introduction}
\label{sec:intro}

Verifying and reasoning about dynamic systems that integrate processes and data
is a long-standing challenge that attracted considerable attention, and that
led to a flourishing series of results, within business process
management~\cite{Reic12,bpm19,bpm20} and data
management~\cite{Vian09,CaDM13,BCDDM13,DeHV14,DeLV19}. Among the several
conceptual models studied in this area, data-centric systems and in particular artifact-centric
systems have been brought forward as a principled approach where relevant
(business) objects are elicited, then defining how actions evolve them
throughout their lifecycle~\cite{Hull08}. Different formal models have been
proposed to capture artifact systems and study their
verification~\cite{CaDM13}.  One of the most studied settings considers
artifact systems as being composed of: a \emph{read-only database} storing
background information about artifacts that does not change during the
evolution of the system; a \emph{working memory}, used to store data that can
be modified in the course of the evolution; and \emph{transitions} (also called
\emph{actions} or \emph{services}) that query the read-only database and the
working memory and use the retrieved answers to update the working memory.
Verification of such systems is particularly challenging, not only because the
working memory in general evolves through infinitely many different
configurations, but also because the desired verification properties should
hold regardless of the specific content of the read-only database, thus calling
for a particular form of parameterised verification
\cite{DaDV12,DeLV19,CalEtAl19,CalEtAl20}.

In this paper, we study for the first time semantic artifact systems where the
read-only database is substituted by a Description Logic ontology, which stores
background, incomplete information about the artifacts.
In this setting, two possible notions of parameterisation may be studied: one where the evolution of the system is verified against all possible choices for the ABox, another where verification is against all possible models of a fixed ABox.
In this work, we adopt the latter hypothesis, and thus verify whether the artifact system enjoys desired properties irrespectively of how the information explicitly provided by the
ABox is completed through the TBox assertions.

More in detail, we consider an extensively studied, basic model of such
artifact-centric systems, called \emph{simple artifact system} (\emph{SAS})
in~\cite{CalEtAl20}, where the artifact working memory consists of a fixed set
of \emph{artifact variables} \cite{DHPV09,DaDV12,CalEtAl20}. On top of this
basis, we study the verification of safety properties in the case where the
ontology is specified in (a slight extension of) RDFS~\cite{W3Crec-RDFS}, a
schema language for the Semantic Web formalized by the W3C, and we make use of
the ontology signature to express the transitions that update the working
memory. For this setting, we show that we can decide safety properties in
\PSpace by relying on an SMT-based backward reachability procedure.

In spirit, our approach is reminiscent of previous works studying the
verification of dynamic systems (in particular, Golog programs) operating over
a DL ontology, such as~\cite{CLLZ14,ZaCl16}. In fact, both in their settings
and ours, the dynamic system evolves each model of the ontology, and
verification properties are assessed over all the resulting evolutions. This is
radically different from approaches where the ABox itself is evolved by the
process, with an execution semantics following Levesque's functional approach,
in which query entailment over the current state is used to compute the
successor states~\cite{BCDD*12,BCMD*13}.  However, we differ
from~\cite{CLLZ14,ZaCl16} in that our goal is not only to derive foundational
results, but also to transfer such results into practical algorithms and thus
obtain a model that is readily implementable by relying on a state-of-the-art
SMT-based model checker such as MCMT~\cite{mcmt10}.
As customary in the formal literature on artifact-centric systems, our approach is based on actions that manipulate the artifact variables, coupled with condition-action rules that declaratively define which actions are currently executable, and with which parameters. Alternative choices could be seamlessly taken, by adapting approaches that rely on an explicit description of the control-flow, e.g., based on state machines \cite{LeoniFM20} or Petri nets interpreted with interleaving semantics \cite{bpm20,MontaliR020}. 

This paper is the extended version of~\cite{DL21}. All the detailed proofs can be found in the appendix.

\section{Preliminaries}
\label{sec:prel}

In this section, we first recall the syntax and semantics of first-order logic
(FO). We then define the syntax of the DL \ourDL considered in this paper,
which is a slight extension of RDFS~\cite{W3Crec-RDFS}.
Its semantics is provided by means of the standard translation, mapping
\ourDL ontologies into equivalent sets of FO formulas.



\subsection{First-Order Logic Preliminaries}

The alphabet of \emph{first-order logic} (\emph{FO}) consists of: countably
infinite and pairwise disjoint sets $\NPr$ of \emph{predicate symbols}
(with $\ar(P) \in \mathbb{N}$ being the arity of $P \in \NPr$),
$\NF$ of \emph{function symbols}
(with $\ar(f) \in \mathbb{N}$ being the arity of $f \in \NF$),
$\NI$ of \emph{individual symbols} (or \emph{individual names}),
and $\Var$ of \emph{variables};
the \emph{equality symbol} `$=$';
the \emph{Boolean operators} `$\lnot$'
and `$\land$';
and the \emph{existential quantifier} `$\exists$'.
An (\textit{FO}) \emph{formula} is an expression
$\p ::= P(\ut) \mid s = t \mid \lnot \p \mid (\p \land \p) \mid \exists x \p$,
where $x \in \Var$, $P \in \NPr$,
$s$, $t$ are terms, and
$\ut = (t_1, \ldots, t_{\ar(P)})$ is a (possibly empty) tuple of terms,
where \emph{terms} are defined inductively as follows:
$t ::= x \mid a \mid f(\ut)$, where $x \in \Var$, $a \in \NI$,
$f \in \NF$, and
$\ut = (t_1, \ldots, t_{\ar(f)})$.  A formula of the form
$P(\ut)$ is called an \emph{atom}, and a \emph{literal} has the form
$P(\ut)$ or $\lnot P(\ut)$.
We adopt the usual abbreviations and conventions: in particular, $\p \lor \q
= \lnot (\lnot \p \land \lnot \q)$ and $\forall x \p = \lnot \exists x \lnot
\p$, where $\forall$ is the \emph{universal quantifier}.
%
We write $\p(\ux)$ to indicate that the \emph{free variables} (defined as
usual) of $\p$ are included in $\ux$, and we write $\p(\ua)$ for the formula
obtained from $\p(\ux)$ by substituting $\ua$ to $\ux$.  Similar notions and
notation are adopted for terms.  A \emph{sentence} is defined as a formula
without free variables, while we call
\emph{quantifier-free} a formula without any occurrence of existential or
universal quantifiers.
A formula is \emph{existential} if it has the form $\exists \ux \p(\ux)$, where
$\p$ is a quantifier-free formula, and it is \emph{universal} if it has the
form $\forall \ux \p(\ux)$, where $\p$ is quantifier-free.
A (\textit{FO}) \emph{theory} $T$ is a set of FO sentences, and $T$ is said to
be \emph{universal} if every $\p \in T$ is universal.  A \emph{signature}
$\Sigma$ is a subset of $\NPr \cup \NF \cup \NI$. For a set $\Gamma$ of
formulas, the \emph{signature of $\Gamma$}, denoted $\sig{\Gamma}$, is the set
of predicate, function, and individual symbols occurring in $\Gamma$.  Given a
signature $\Sigma$, we say that $\Gamma$ is a set of \emph{$\Sigma$-formulas}
if $\sig{\Gamma} = \Sigma$ (we will use \emph{$\Sigma$-formula},
\emph{$\Sigma$-theory}, etc., in an analogous way).
%


An (\textit{FO}) \emph{interpretation} is a pair
$\Imc =(\dom, \cdot^{\Imc})$, where
$\dom$ is a non-empty set, called
\emph{domain} of $\Imc$, and $\cdot^{\Imc}$ is an \emph{interpretation
 function} such that:
$\Int{P} \subseteq (\dom)^{\ar(P)}$, for every $P \in \NP$;
$\Int{f} \colon (\dom)^{\ar(f)} \longrightarrow \dom$, for every $f \in \NF$;
and $\Int{a} \in \dom$, for every $a\in\NI$.
An \emph{assignment} in $\Imc$ is a function
$\assign \colon \Var \longrightarrow \dom$.
We define the \emph{value} of a term $t$ in $\Imc$ under $\assign$ as follows:
$\assign(t) = \assign(x)$, if $t = x$;
$\assign(t) = \Int{a}$, if $t = a \in \NI$;
and
$\assign(t) = \Int{f}(\assign(\ut))$,
if $t = f(\ut)$, where $f \in  \NF$ and, for an $m$-tuple
$\ut = (t_1, \ldots, t_m)$ of terms, we set
$\assign(\ut) = (\assign(t_1), \ldots, \assign(t_m))$.
The notion of a formula $\p$ being \emph{satisfied} in an interpretation $\Imc$
under an assignment $\assign$, or of $\Imc$ being a \emph{model} of $\p$ under
$\assign$, written $\Imc \models^{\assign} \p$, is inductively defined as:
\[
  \begin{array}{l@{~~\text{iff}~~}l}
    \Imc \models^{\assign} P(\ut) & \assign(\ut) \in \Int{P},\\
    \Imc \models^{\assign} s = t  & \assign(s) = \assign(t),\\
    \Imc  \models^{\assign} \lnot \q & \text{not } \Imc \models^{\assign} \q,\\
    \Imc  \models^{\assign} \q \land \chi
    & \Imc \models^{\assign} \q \text{ and } \Imc \models^{\assign} \chi,\\
    \Imc  \models^{\assign} \exists x \q
    & \Imc \models^{\assign'} \q \ \text{for some } \assign'
    \text{ that can differ from } \assign \text{ on } x.
 \end{array}
\]
For a formula $\p(\ux)$, we write $\Imc \models \p[\ud]$ in place of
$\Imc \models^{\assign} \p(\ux)$, with $\assign(\ux) = \ud$.
%
We say that a set $\Gamma$ of formulas is \emph{satisfied} in an interpretation
$\Imc$ under an assignment $\assign$, or that $\Imc$ is a \emph{model} of
$\Gamma$ under $\assign$, written $\Imc \models^{\assign} \Gamma$, if
$\Imc \models^{\assign} \p$, for every $\p \in \Gamma$ (we refer to a singleton
$\Gamma = \{\p\}$ simply as $\p$).
%
%
For a sentence $\p$, the satisfaction of $\p$ in $\Imc$ under $\assign$ does
not depend on $\assign$, thus we write $\Imc \models \p$ in place of
$\Imc \models^{\assign} \p$, and we say that $\p$ is satisfied in $\Imc$.
For a theory $T$, we say that $T$ is \emph{satisfied} in an interpretation
$\Imc$ (or that $\Imc$ is a \emph{model} of $T$), written $\Imc \models T$, if
every sentence of $T$ is satisfied in $\Imc$.
A formula $\p$ is \emph{satisfiable w.r.t.\ $T$} (or \emph{$T$-satisfiable}) if
there exist an interpretation $\Imc$ and an assignment $\assign$ in $\Imc$ such
that $\Imc \models T$ and $\Imc \models^{\assign} \p$.
Moreover, we say that $T$ \emph{logically implies} a formula $\p$, or that $\p$
is a \emph{logical consequence} of $T$, written $T \models \p$, if, for every
interpretation $\Imc$ and every assignment $\assign$ in $\Imc$,
$\Imc \models T$ implies that $\Imc \models^{\assign} \p$.
Finally, formulas $\p$, $\q$ are \emph{equivalent w.r.t.\ $T$} (or
\emph{$T$-equivalent}) if $T \models \p \leftrightarrow \q$.

\subsection{Description Logics Preliminaries}


Let $\NC$, $\NR$, and $\NI$ be countably infinite and pairwise disjoint sets of
\emph{concept}, \emph{role}, and \emph{individual names}, respectively (with
$\NC\cup \NR \subseteq \NPr$, i.e., concept and role names are predicate
symbols, with arity~1 and~2, respectively).

The DL we consider here is an extension of RDFS~\cite{W3Crec-RDFS} with
disjointness between concepts and roles, conjunction and (one-level) qualified
existential quantification on the left-hand side of inclusions, and inclusion
of direct and inverse roles.  We denote such DL \ourDL, and we define it below.

In \ourDL, \emph{concepts} $C$ and \emph{roles} $R$ are defined according to
the grammar
\[
  \begin{array}{r@{~~}c@{~~}l}
    R &::=& P ~\mid~ P^{-},\\
    C &::=& A_1 \AND \cdots \AND A_n ~\mid~ \SOMET{R} ~\mid~ \SOME{R}{A},
  \end{array}
\]
where $P \in \NR$, $n\geq 1$, and $A,A_1,\ldots,A_n \in \NC$.

A \emph{concept inclusion (CI)} has the form $C \ISA A$ or $C \ISA \NOT A$, and
a \emph{role inclusion (RI)} has the form $R \ISA R'$ or $R \ISA \NOT R'$,
where $C$ is an \ourDL concept, $A\in\NC$, and $R$, $R'$ are roles.
An \ourDL \emph{TBox} $\Tmc$ is a finite set of CIs and RIs.
An \emph{assertion} has the form $A(a)$, $\lnot A(a)$, $P(a,b)$,
$\lnot P(a,b)$, $(a=b)$, or $\lnot(a=b)$, where $A \in \NC$, $P \in \NR$, and
$a,b\in\NI$.  An \emph{ABox} $\Amc$ is a finite set of assertions.  (We point
out that in an ABox we allow for negated assertions, which is a feature that is
not always supported in DLs.)
%
An \ourDL \emph{ontology} $\Omc$ is a pair $(\Tmc,\Amc)$, where $\Tmc$ is a
TBox and $\Amc$ is an ABox.

We observe that \ourDL is incomparable in expressive power with the DLs of the
popular \dllite family~\cite{CDLLR07,ACKZ09}.  Indeed, while \dllite allows for
the use of existential quantification on the right-hand side of CIs, these are
ruled out in \ourDL.  On the other hand, in \ourDL one can locally type the
second component of a role through the use of qualified existential
quantification on the left-hand side of CIs, while this is not possible in
\dllite.  As we will see later, differently from what happens for \dllite, the
FO
translation of an \ourDL ontology is a universal theory.

\begin{example}
  \label{ex:ont}
  To represent part of the domain knowledge on job hiring processes for
  university personnel, we define the \ourDL ontology $\Omc=(\Tmc,\Amc)$, where
  $\Tmc$ consists of the following concept inclusions:
  \[
    \begin{array}{rcl@{\qquad}rcl}
      \ex{AcademicPosition} &\ISA& \ex{JobPosition}
      & \ex{AcademicPosition} &\ISA& \NOT\ex{AdminPosition}\\
      \ex{AdminPosition}  &\ISA& \ex{JobPosition}
      & \ex{User} &\ISA& \NOT\ex{JobPosition}\\
      \SOMET{\ex{appliesFor}} &\ISA& \ex{User}
      & \SOMET{\ex{appliesFor}^-} &\ISA& \ex{JobPosition}\\
      \SOMET{\ex{suitableFor}} &\ISA& \ex{User}
      & \SOMET{\ex{suitableFor}^-} &\ISA& \ex{JobPosition}\\
      \SOMET{\ex{suitableFor}} &\ISA& \ex{PositivelyEvaluated}
      & \ex{EligibleUser} &\ISA& \ex{User}\\
      \ex{User} \AND \ex{Graduate} &\ISA& \ex{EligibleUser}
      & \ex{EligibleUser} &\ISA& \ex{Graduate}
    \end{array}
  \]
  while $\Amc$, which stores data on available job positions, contains the
  assertions
  \[
    \begin{array}{l}
      \ex{AcademicPosition}(\ex{professor}_{\ex{123}}),\\
      \ex{AcademicPosition}(\ex{researcher}_{\ex{123}}),
    \end{array}
    \qquad
    \begin{array}{l}
      \ex{AdminPosition}(\ex{secretary}_{\ex{123}}),\\
      \ex{AdminPosition}(\ex{secretary}_{\ex{456}}).
    \end{array}
  \]
  Moreover, we assume that $\Amc$ contains all the assertions of the form
  $\lnot A(\ex{u})$, $\lnot P(\ex{u}, a)$ and $\lnot P(a, \ex{u})$, for a
  distinguished individual name $\ex{u} \in \NI$ and every
  $A, P, a \in \sig{\Omc}$, so that $\ex{u}$ can be used to represent an
  \emph{undefined value}.
  The CIs of $\Tmc$ formalise the following facts:
  \begin{inparablank}
  \item there are both academic and administrative job positions and these are
    disjoint;
  \item users and job positions are disjoint;
  \item \ex{appliesFor} relates users to job positions;
  \item to be suitable for
    something one has to be a user that is positively evaluated;
  \item the range of \ex{suitableFor} is included in the extension of
    \ex{JobPositions};
  \item an eligible user is defined as a graduate user.
  \end{inparablank}
  \envend
\end{example}


We define now the \emph{standard translation} from \ourDL expressions to FO
formulas, which maps concepts to FO formulas with one free variable, and roles
to FO formulas with two free variables.  Specifically, the translation $T$
generates formulas that contain just two variables $x,y\in\Var$, and is defined
as follows:
\[
  \begin{array}{rcl@{\qquad\qquad\qquad}rcl}
    \sttr{A_1\AND\cdots\AND A_n} &=&
    \multicolumn{4}{l}{A_1(x)\land\cdots\land A_n(x),} \\
    \sttr{P} &=& P(x,y),
    &  \sttr{P^-} &=& P(y,x),\\
    \sttr{\SOMET{R}} &=& \exists y \sttr{R},
    & \sttr{\SOME{R}{A}} &=& \exists y (\sttr{R} \land A(y)),\\
    \sttr{\NOT A} &=& \lnot \sttr{A},
    & \sttr{\NOT R} &=& \lnot \sttr{R},
  \end{array}
\]
where $A,A_1,\ldots,A_n$ are unary predicates and $P$ is a binary predicate.
Moreover, we map CIs and RIs into universal FO sentences in the following way:
\[
  \sttr{C \ISA D} = \forall x (\sttr{C} \to \sttr{D}), \qquad\qquad
  \sttr{R \ISA S} = \forall x \forall y (\sttr{R} \to \sttr{S}),
\]
where $D$ stands for either $A$ or $\NOT A$, and $S$ stands for either $R'$ or
$\NOT R'$.
We also set
$\sttr{\Tmc} = \bigcup_{\beta \in \Tmc} \{ \sttr{\beta} \}$.
Assertions $\alpha$ are (identically) mapped into FO literals without free
variables (i.e., \emph{ground}), as $\sttr{\alpha} = \alpha$, and we set
$\sttr{\Amc} = \bigcup_{\alpha \in \Amc} \{ \sttr{\alpha} \}$.
Finally, $\sttr{\Omc} = \sttr{\Tmc} \cup \sttr{\Amc}$.
It is easy to see that the set of FO sentences obtained as the translation
$\sttr{\Omc}$ of an \ourDL ontology $\Omc$, can be equivalently rewritten into
a \emph{universal Horn theory}~\cite{KonEtAl14,Hod93}.
Such a theory, which we identify with $T(\Omc)$, can be obtained from $T(\Omc)$
by simply putting formulas into prenex normal form.
%


The semantics for \ourDL expressions can be given in terms of their FO
translation~\cite{KonEtAl14}.
For an interpretation $\Imc=(\dom,\cdot^{\Imc})$ and a concept $C$, we define
the \emph{extension} of $C$ in $\Imc$ as
$\Int{C}=\{ d \in \dom \mid \Imc \models \sttr{C}[d] \}$.  Similarly, for a
role $R$, we define its extension in $\Imc$ as
$\Int{R} = \{ (d,e) \in \dom\times\dom \mid \Imc \models \sttr{R}[d,e] \}$.  We
say that $C$ and $R$ are \emph{satisfied} in $\Imc$ if $\Int{C}\neq\emptyset$
and $\Int{R}\neq \emptyset$, respectively.

Moreover, given a CI, RI, assertion, TBox, ABox, or ontology $\Gamma$, we say
that $\Gamma$ is \emph{satisfied} in $\Imc$ (or that $\Imc$ is a \emph{model}
of $\Gamma$), written $\Imc \models \Gamma$, iff $\Imc \models \sttr{\Gamma}$.
Given an ontology $\Omc$ and (a concept, role, CI, RI, or assertion mapped, via
its FO translation, into) an FO formula $\p$, we say that $\p$ is
\emph{satisfiable w.r.t.\ $\Omc$} (or \emph{$\Omc$-satisfiable}) if there
exists a model $\Imc$ of $\Omc$ that satisfies $\p$ under some assignment in
$\Imc$.
Finally, we say that $\Omc$ \emph{logically implies} an FO formula $\p$, or
that $\p$ is a \emph{logical consequence} of $\Omc$, written $\Omc \models \p$,
if, for every model $\Imc$ of $\Omc$ and every assignment $\assign$ in $\Imc$,
we have that $\Imc$ satisfies $\p$ under $\assign$.


\section{Basic Model-Theoretic Properties}
\label{sec:firstres}

In this section, we prove the model-theoretic properties that will be used
later on to develop our verification machinery.
Specifically, we show here that the standard
translation of
the \ourDL ontologies introduced in the previous section
admits model completion,
and has the constraint satisfiability problem decidable.
These properties will allow us, in the subsequent sections, to verify
suitably defined
DL-based data-aware processes by employing a variant of the SMT-based backward reachability procedure introduced in \cite{CalEtAl19}.
To present our results, we
require the following preliminary notions.

A formula that is a conjunction of $\Sigma$-literals is called a \emph{$\Sigma$-constraint}.
Given a $\Sigma$-theory $T$, we define the \emph{constraint satisfiability problem for $T$} as follows:
given a formula $\exists \underline{y} \p(\underline{x}, \underline{y})$, where $\p(\underline{x}, \underline{y})$ is a $\Sigma$-constraint, decide whether
$\exists \underline{y} \p(\underline{x}, \underline{y})$ is satisfiable w.r.t. $T$.
%
%
A theory $T$ has \emph{quantifier elimination} iff, for every $\sig{T}$-formula $\p(\underline{x})$, there exists a quantifier-free formula $\psi(\underline{x})$ such that $T \models \p(\underline{x}) \leftrightarrow \psi(\underline{x})$.
%
Finally, we will use the following definition of model completion, which is restricted to cover the case of universal theories (the ones considered in this work) and that is nonetheless known to be equivalent (for universal theories) to the usual one from model theory~\cite{ChaKei90,Ghi04,CalEtAl19}.
Let $T$ be a universal $\Sigma$-theory and let $T^{\ast} \supseteq T$ be a
further $\Sigma$-theory. We say that $T^{\ast}$ is a \emph{model completion of
 $T$} iff:
\begin{inparaenum}[\itshape (i)]
\item every $\Sigma$-constraint satisfied in some model of $T$ is also
  satisfied in some model of $T^{\ast}$;
\item $T^{\ast}$ has quantifier elimination.
\end{inparaenum}

We now state the main technical result of the section.


\begin{restatable}{theorem}{rldhlass}
\label{prop:rldhlass}
Given an \ourDL ontology $\Omc$, $\sttr{\Omc}$ is a finite universal FO theory
that
\begin{inparaenum}[\itshape (i)]
\item has a decidable constraint satisfiability problem, and
\item admits a model completion.
\end{inparaenum}
\end{restatable}

\begin{proof}[Proof (Sketch)]
To prove Point~\textit{(i)}, we reduce to $\ourDL$ (seen as a fragment of
$\mathcal{ALCHI}$,~\cite{Rud11}) ontology satisfiability. For Point~\textit{(ii)}, since
there is no function symbol in $\Sigma_{T(\Omc)}$, it is sufficient to show
that $T(\Omc)$ enjoys the amalgamation property: this is proved by explicitly
constructing a $T(\Omc)$-amalgam for every pair of models $\Imc_{1}$ and
$\Imc_{2}$ of $T(\Omc)$ sharing a submodel $\Imc_0$. 
 See the appendix for details.
\end{proof}

%

\begin{remark}
\label{rem:qe}
For every \ourDL ontology $\Omc$,
the model completion $T(\Omc)^{\ast}$ of $T(\Omc)$ admits quantifier
elimination. The algorithm for quantifier elimination in $T(\Omc)^{\ast}$
follows from the proof of Theorem~\ref{prop:rldhlass}:
to eliminate $\exists x$ from a $\Sigma_{T(\Omc)}$-formula
$\exists x \varphi(x, \uy)$, it is sufficient to take the conjunction of the
clauses $\chi(\uy)$ implied by $\varphi(x, \uy)$, which are finitely many for
$T(\Omc)$, up to $T(\Omc)$-equivalence.
This procedure is used in Algorithm~\ref{fig:algorithm} below and is
crucial to get the decidability results of Theorem~\ref{prop:dlsasdec}.
\envend
\end{remark}

Properties~\textit{(i)} and~\textit{(ii)} from Theorem~\ref{prop:rldhlass}
are in line with the foundational framework of~\cite{CalEtAl19,CalEtAl20},
where a third property is additionally assumed: the \emph{finite model property
 for constraint satisfiability} (see the references for the definition).
However, differently from~\cite{CalEtAl19,CalEtAl20}, this property is not
needed anymore for the results of our paper.
This is an important difference from~\cite{CalEtAl19,CalEtAl20}, since the
artifact systems studied there require to interact with \emph{finite
 structures} (i.e., databases), whereas in the context of the present work we
admit that the models of the knowledge base of our artifact systems can be
infinite.

\section{Ontology-Based Data-Aware Processes}
\label{sec:dlsas}

In this section, we present our main contributions. We first define our model,
called \ourDL-based simple artifact systems, or \ourDL-SASs for short, to
formalise data-aware processes under \ourDL ontologies.  These systems are a
variant of the artifact-centric systems studied in~\cite{CalEtAl19}.
\ourDL-SASs read data from a given \ourDL ontology, used to store background
information of the system, and manipulate individual variables, called
\emph{artifact variables}, which represent the current state of the process.
We then study the parameterised safety problems of such models by adopting a
symbolic version~\cite{lmcs2010,CalEtAl20} of the well-known backward reachability procedure~\cite{AbdullaCJT96}.

\subsection{DL-Based Simple Artifact Systems}

We first require the following preliminary notions.  For an \ourDL ontology
$\Omc$, an \emph{$\Omc$-partition} is a finite set
$P = \{ \kappa_{1}(\ux), \ldots, \kappa_{n}(\ux) \}$ of $\sig{\Omc}$-literals
such that
$\Omc \models \forall \ux \bigvee^{n}_{i = 1} \kappa_{i}(\ux) \land \forall \ux
\bigwedge_{i \neq j} \lnot ( \kappa_{i}(\ux) \land \kappa_{j}(\ux) )$.
%
Given an ontology $\Omc$, an $\Omc$-partition $P = \{ \kappa_{1}(\ux), \ldots, \kappa_{n}(\ux) \}$, and
$\sig{\Omc}$-terms $\ut(\ux) = (t_{1}(\ux), \ldots, t_{n}(\ux) )$, (the value
of) a \emph{case-defined function $F$ based on $P$ and $\ut$}, for a fresh
function symbol $F \in \NF$,
is defined as follows: for every model $\Imc$ of $\Omc$, every assignment
$\assign$ in $\Imc$, and every tuple $\ux$ of variables,
$\assign(F(\ux)) = \assign(t_{i}(\ux))$, if $\Imc \models^{\assign} \kappa_{i}(\ux)$.

In order to introduce verification problems in a symbolic setting, one first
has to specify which formulas are used to represent
\begin{inparaenum}[\itshape (i)]
\item the sets of states,
\item the system initialisations, and
\item the system evolution.
\end{inparaenum}
To capture these aspects, we provide the following definitions.

An \emph{\ourDL-based simple artifact system (\ourDL-SAS)} is a tuple
\[
  \textstyle
  \Smc = (\Omc, \ux, \iota(\ux), \bigcup_{j = 1}^{m} \{ \tau_{j} (\ux, \ux') \}),
\]
where $m \in \mathbb{N}$, and
\begin{itemize}
\item $\Omc = (\Tmc, \Amc)$ is an \ourDL ontology;
\item $\ux = (x_{1}, \ldots, x_{n})$ is a tuple of variables, called
  \emph{artifact variables}, and $\ux'$ is a tuple of variables that are
  renamed copies of variables in $\ux$;
\item $\iota(\ux) = \bigwedge_{i = 1}^{n} x_{i}  = a_{i}$, with $a_{i} \in \NI$,
  is an \emph{initial state} formula;
\item $\tau_{j} (\ux, \ux') = \exists \uy (\gamma^{j}(\ux, \uy) \land
  \bigwedge_{i = 1}^{n} x'_{i}  = F^{j}_{i}(\ux, \uy))$,
  for $1 \leq j \leq m$, is a \emph{transition formula}, where $\gamma^{j}(\ux,
  \uy)$ is a
  conjunction of $\sig{\Omc}$-literals
  called \emph{guard} of $\tau_{j}$, and $x'_{i} = F^{j}_{i}(\ux, \uy)$, where
  each $F^{j}_{i}$ is a case-defined function based on some $\Omc$-partition
  and list of $\sig{\Omc}$-terms, is an \emph{update} of $\tau_{j}$.
\end{itemize}

Given an \ourDL ontology $\Omc$, we call \emph{state ($\Sigma_{\Omc}$-)formula}
a quantifier-free $\Sigma_{\Omc}$-formula $\p(\ux)$. A state formula constrains
the content of the artifact variables characterising the current states of the
systems. Notice that a state formula can represent a (possibly infinite) set of
states, because of the presence of (possibly infinitely many) different
elements in a model of the ontology $\Omc$.
%
A \emph{safety formula} for $\Smc$ is a state $\Sigma_{\Omc}$-formula $\nu(\ux)$, used to describe the undesired states of the system.
We say that $\Smc$ is \emph{safe w.r.t.\ $\nu(\ux)$} if there does
not exist $k \geq 0$ and a formula of the form
\begin{equation}
  \tag{$\star$}
  \label{eq:safetyform}
  \iota(\ux^{0}) \land \tau_{j_{0}}(\ux^{0}, \ux^{1}) \land \cdots \land \tau_{j_{k-1}}(\ux^{k-1}, \ux^{k}) \land \nu(\ux^{k}),
\end{equation}
that is satisfiable w.r.t.\ $\Omc$, where
$1 \leq j_{0}, \ldots, j_{k-1} \leq m$ and each $\ux^h$, with
$0 \leq h \leq k$, is a tuple of variables that are renamed copies of variables
in $\ux$.
The \emph{safety problem for $\Smc$} is the following decision problem: given a
safety formula $\nu(\ux)$ for $\Smc$, decide whether $\Smc$ is safe w.r.t.\
$\nu(\ux)$.
This verification problem is \emph{parametric} on the models of a fixed \ourDL
ontology, since safety is assessed \emph{irrespectively of the choice of such a
 model}. This implies that, when the system is safe, it is so for \emph{every}
execution of the process under \emph{every} possible model (which in principle
are infinitely many) of the given ontology.

\begin{example}
\label{ex:sas}
We develop a simplified job hiring process for university personnel based on
the domain knowledge formalised in Example~\ref{ex:ont}. Each application is
created using a dedicated website portal, where users that are potentially
interested in applying need to register in advance. When a registered user
decides to apply, the data created by this single application do not have to be
stored persistently and thus can be maintained just by using artifact variables
(described below) that can interact with the knowledge base.  All these
variables are initialised with an undefined value $\ex{u}$. In the first
transition of the system, an application is created by a registered user, which
falls into
the extension of the concept \ex{User}: the information about this user is then
stored in the artifact variable $x_{\ex{applicant}}$. At this point, the
application website asks the user whether they hold a university degree: in
case of an affermative answer, the website accepts the user as eligible, the
information about the user is stored using $x_{\ex{applicant}}$ and the process
can progress. Then, the user picks up a job position (assigned to
$x_{\ex{job}}$) and applies for it. The following steps of the process involve
the evaluation of the application: for both academic and administrative
positions, if the eligible candidate is suitable for the chosen position, they
are declared winner (assigned to $x_{\ex{winner}}$), otherwise they are
declared loser (assigned to $x_{\ex{loser}}$). To formalise this process, we
define the \ourDL-SAS
$\Smc = (\Omc, \ux, \iota(\ux), \bigcup_{j = 1}^{7} \{ \tau_{j} (\ux, \ux')
\})$ so that:
\begin{itemize}
\item the ontology $\Omc$ is the \ourDL ontology given in Example~\ref{ex:ont};
\item the artifact variables are $\ux = (x_{\ex{applicant}}, x_{\ex{job}}, x_{\ex{eligible}}, x_{\ex{winner}}, x_{\ex{loser}} )$;
\item the initial state formula is

  ~~$
    \iota = ( x_{\ex{applicant}} = \ex{u}) \land (x_{\ex{job}} = \ex{u}) \land (x_{\ex{eligible}} = \ex{u}) \land (x_{\ex{winner}} = \ex{u}) \land (x_{\ex{loser}} = \ex{u} );
  $
\item the transition formulas are

  $\begin{array}{rcl}
     \tau_1 &=& \exists y_1 (\ex{User}(y_1) \land x'_{\ex{applicant}} = y_1),\\
     \tau_2 &=& \ex{EligibleUser}(x_{\ex{applicant}}) \land x'_{\ex{eligible}} = x_{\ex{applicant}},\\
     \tau_3 &=& \exists z_1 (\ex{JobPosition} (z_1) \land \ex{appliesFor}(x_{\ex{eligible}}, z_1) \land x'_{\ex{job}} = z_1),\\
     \tau_4 &=& \ex{AcademicPosition} (x_{\ex{job}}) \land \ex{suitableFor}(x_{\ex{eligible}}, x_{\ex{job}}) \land x'_{\ex{winner}} = x_{\ex{eligible}},\\
     \tau_5 &=& \ex{AdminPosition} (x_{\ex{job}}) \land \ex{suitableFor}(x_{\ex{eligible}}, x_{\ex{job}}) \land x'_{\ex{winner}} = x_{\ex{eligible}},\\
     \tau_6 &=& \ex{AcademicPosition} (x_{\ex{job}}) \land \lnot\ex{suitableFor}(x_{\ex{eligible}}, x_{\ex{job}}) \land x'_{\ex{loser}} = x_{\ex{eligible}},\\
     \tau_7 &=& \ex{AdminPosition} (x_{\ex{job}}) \land \lnot \ex{suitableFor}(x_{\ex{eligible}}, x_{\ex{job}}) \land x'_{\ex{loser}} = x_{\ex{eligible}}.
   \end{array}$
\end{itemize}
An undesired situation of the system is the one where an applicant registered
user is declared winner even if they were not eligible. This situation is
formally described by the following safety formula for $\Smc$:
\nobrackettag
\[
  \nu = \ex{User}(x_{\ex{winner}}) \land \lnot\ex{EligibleUser}(x_{\ex{winner}}).
  \tag{\envendmark}
\]
\end{example}

\subsection{Backward Search for \ourDL-SASs}

\begin{algorithm}[t]
\label{alg:backsearch}
\SetKwProg{Fn}{Function}{}{end}
\Fn{$\mathsf{BReach}(\nu)$}{
\setcounter{AlgoLine}{0}
\ShowLn$\phi\longleftarrow \nu$;  $B\longleftarrow \bot$\;
\ShowLn\While{$\phi\land \neg B$ is $T(\Omc)$-satisfiable}{
\ShowLn\If{$\iota\land \phi$ is $T(\Omc)$-satisfiable}
{\textbf{return}  ($\mathsf{unsafe}$, $\textit{unsafe trace of form~\eqref{eq:safetyform}}$);}
\setcounter{AlgoLine}{3}
\ShowLn$B\longleftarrow \phi\vee B$\;
\ShowLn$\phi\longleftarrow \Pre(\tau, \phi)$\;
\ShowLn$\phi\longleftarrow \QE(T(\Omc)^{\ast},\phi)$\;
}
\textbf{return} $\mathsf{safe}$;}{
\caption{SMT-based backward reachability procedure}
\label{fig:algorithm}
}
\end{algorithm}

Algorithm~\ref{alg:backsearch}
shows the \emph{SMT-based backward reachability procedure} (or
\emph{backward search}) for handling the safety problem for an \ourDL-SAS $\Smc$.  An integral
part of the algorithm is to compute \textit{symbolic} preimages (Line~5). The intuition behind the algorithm is to execute a loop where, starting from the undesired states of the system (described by the safety formula $\nu(\ux)$), the state space of the system is explored \emph{backward}: in every iteration of the while loop (Line~2), the current set of states is \emph{regressed} through transitions thanks to the preimage computation.
For that purpose,  for any $\tau(\uz,\uz')$ and $\phi(\uz)$ (where $\uz'$ are renamed copies of $\uz$),
we define
$\tau:= \bigvee_{h=1}^{m} \tau_h$ and  $\textit{Pre}(\tau,\phi)$ as the formula
$\exists \uz'(\tau(\uz, \uz')\land \phi(\uz'))$.
Let $\phi(\ux)$ be a state formula, describing the state of the artifact variables $\ux$.
The \emph{preimage} of the set of states described by the formula
$\phi(\ux)$ is the set of states described by
$\textit{Pre}(\tau,\phi)$ (notice that, when $\tau=\bigvee\hat\tau$,
        then $\textit{Pre}(\tau,\phi)=\bigvee\textit{Pre}(\hat\tau,\phi)$). We recall that a state formula is a quantifier-free $\Sigma_{\Omc}$-formula. Unfortunately, because of the presence of the existentially quantified variables $\uy$ in $\tau$, $\textit{Pre}(\tau,\phi)$ is \emph{not} a state formula, in general.
As stated in~\cite{CalEtAl19,CalEtAl20}, if the quantified variables were not
\emph{eliminated},
we would break the \emph{regressability} of the procedure: indeed, the states
reached by computing preimages, intuitively described by
$\textit{Pre}(\tau,\phi)$, need to be represented by a state formula $\phi'$ in
the new iteration of the while loop. In addition, the increase of the number of
variables due to the iteration of
the preimage computation would affect the performance of the satisfiability
tests described below, in case the loop is executed many times. In order to
solve these issues, it is essential to introduce the subprocedure
$\QE(T(\Omc)^{\ast},\phi)$ in Line~6.

$\QE(T(\Omc)^{\ast},\phi)$ in Line~6 is a subprocedure that implements the
quantifier elimination algorithm of $T(\Omc)^{\ast}$ and that converts the
preimage $\Pre(\tau,\phi)$ of a state formula $\phi$ into a state formula
(equivalent to it modulo the
axioms of $T(\Omc)^{\ast}$), so as to guarantee the regressability of the
procedure: this conversion is possible since $T(\Omc)^{\ast}$ eliminates from
$\tau_h$ the existentially quantified variables $\uy$.
Backward search computes iterated preimages of the safety formula $\nu$,
until a fixpoint is reached (in that case, $\Smc$ is \emph{safe} w.r.t.\ $\nu$)
or until a set intersecting the initial states (i.e., satisfying $\iota$) is
found (in that case, $\Smc$ is \emph{unsafe} w.r.t.\ $\nu$). 
%
%
\textit{Inclusion} (Line~2) and \textit{disjointness}
(Line~3) tests can be discharged via proof obligations to be handled by SMT
solvers.
The fixpoint is reached when the test in Line~2 returns \textit{unsat}: the
preimage of the set of the current states is included in the set of states
reached by the backward search so far (represented as the iterated application
of preimages to the safety formula $\nu$).
The test at Line~3 is satisfiable when the states visited so far by the
backward search includes a possible initial state (i.e., a state satisfying
$\iota$). If this is the case, then $\Smc$ is unsafe w.r.t.\ $\nu$. Together
with the unsafe outcome, the algorithm also returns an unsafe trace of the
form~\eqref{eq:safetyform}, explicitly witnessing the sequence of transitions
$\tau_h$ that, starting from the initial configurations, lead the system to a
set of states satisfying the undesired conditions described by $\nu(\ux)$.
%
%

\begin{restatable}{theorem}{soundcomplete}
\label{thm:soundcomplete}
Backward search
(Algorithm~\ref{alg:backsearch})
is correct for detecting whether an \ourDL-SAS $\Smc$ is safe w.r.t.\ $\nu(\ux)$.
\end{restatable}
\begin{proof}[Proof (Sketch)]
First, we require the following claim, which follows immediately from the
definitions.

\begin{claimn}
  \label{cla:tracesat}
  For every safety formula $\nu(\ux)$ for $\Smc$ and every $k \geq 0$, a
  formula $\vartheta$ of the form~\eqref{eq:safetyform} is satisfiable w.r.t.\
  $\Omc$ iff $\vartheta$ is satisfiable w.r.t.\ $T(\Omc)$.
\end{claimn}
Then, we need to show that, instead of considering satisfiability of formulas
of the form~\eqref{eq:safetyform} in models of $T(\Omc)$, we can concentrate on
satisfiability w.r.t.\ $T(\Omc)^{\ast}$ ($T(\Omc)^{\ast}$ exists thanks to
Property~\textit{(ii)} of Theorem~\ref{prop:rldhlass}).
Then, by exploiting the
algorithm for quantifier elimination in $T(\Omc)^{\ast}$ described in
Remark~\ref{rem:qe}, formulas of the form~\eqref{eq:safetyform} can be
represented via backward search by using quantifier-free formulas. We finally
conclude by noticing that safety\slash unsafety of $\Smc$ w.r.t $\nu(\ux)$ can
be now detected invoking the satisfiability tests (which are effective thanks
to Property~\textit{(i)} of Theorem~\ref{prop:rldhlass}) over those
quantifier-free formulae.
\end{proof}

Backward search for generic artifact systems is not guaranteed to terminate~\cite{CalEtAl20}.
However,
in case $\Smc$ is \emph{unsafe} w.r.t.\ $\nu(\ux)$, an unsafe trace---which is finite---is found after finitely many iterations of the while loop: hence, in the unsafe case, backward search
must terminate.
Together with the theorem above, this
means that the backward reachability procedure is at least a semi-decision procedure for detecting  unsafety of \ourDL-SASs.
Nevertheless,
we show in the following theorem that, in case of
\ourDL-SASs,
backward search \emph{always terminates}:
thus,
it is a full decision procedure,
for which we also
provide a $\PSpace$
upper bound.

\begin{restatable}{theorem}{dlsasdec}
\label{prop:dlsasdec}
For an \ourDL ontology $\Omc$
and an \ourDL-SAS $\Smc = (\Omc, \ux, \iota(\ux), \bigcup_{j = 1}^{m} \{ \tau_{j} (\ux, \ux') \})$,
the safety problem for $\Smc$ is decidable in $\PSpace$ in the combined size of $\ux$, $\iota(\ux)$ and $\bigcup_{j = 1}^{m} \{ \tau_{j} (\ux, \ux') \}$.
\end{restatable}

\begin{proof}[Proof (Sketch)]
For every \ourDL ontology $\Omc$, there are only finitely many quantifier-free
$\Sigma_{T(\Omc)}$-formulas, up to $T(\Omc)$-equivalence, that can be built out
of a finite set of variables $\ux$. Thanks to the availability of the
quantifier elimination procedure $\QE(T(\Omc)^{\ast},\p)$, the overall number
of variables in $\p$ is never increased. This implies that globally there are
only finitely many quantifier-free $\Sigma_{T(\Omc)}$-formulas that
Algorithm~\ref{fig:algorithm} needs to analyse. Hence,
Algorithm~\ref{fig:algorithm} terminates. Concerning complexity, we first note
that the translation
$T(\Omc)$ requires polynomial time. Then, we need to eliminate the occurrences
of case-defined functions (creating an equivalent SAS whose size is polynomial
in the size of the original one), and to modify Algorithm~\ref{fig:algorithm}
by making it nondeterministic with an $\NPSpace$ complexity. The claim follows
by applying Savitch's Theorem.
\end{proof}

We observe that Algorithm~\ref{fig:algorithm} is not yet implemented in the state-of-the-art model checker MCMT (Model Checker Modulo Theories [21]), which is based on SMT solving. Such an implementation, however, can be obtained by extending MCMT with the quantifier elimination algorithm for $T(\Omc)^{\ast}$ (described in Remark~\ref{rem:qe}), required in Line 6, together with a procedure for RDFS+ ontology satisfiability (seen as a fragment of $\mathcal{ALCHI}$,~\cite{Rud11}), required in Lines 2 and 3.

%
%
%

\section{Conclusions}
\label{sec:conc}

We have studied the problem of verification of data-aware processes under
\ourDL ontologies, where the process component can interact with a knowledge
base specified by means of the DL \ourDL,
underpinning the RDFS constructs.
We addressed this problem by introducing a suitable model of DL-based artifact-centric systems, called \ourDL-based SASs, and by leveraging the SMT-based version of the backward reachability procedure, which is a well-known technique to employ for verifying systems of this kind. Specifically, we showed that this procedure is a full decision procedure for detecting safety of \ourDL-based SASs, and we also provided a $\PSpace$ complexity upper bound. 

This work opens several directions for future work.  First, we notice that the
choice of \ourDL ontologies is not intrinsic to our approach.  Indeed,
motivated by conceptual modelling and data integration issues in OBDA
applications, we are currently working on the \emph{\dllite family} of DLs, to
define suitable \dllite-based SASs with analogous decidability and complexity
results.
The main difference we have to account for is that, for a \dllite ontology
$\Omc$, we have an \emph{equisatisfiable} (but not equivalent) translation into
a universal one-variable FO sentence $T(\Omc)$, and Claim~\ref{cla:tracesat} in
the proof of Theorem~\ref{thm:soundcomplete} has to be modified to show that a
trace $\vartheta$ is satisfiable w.r.t.\ $\Omc$ iff a suitably translated trace
$\hat{\vartheta}$ is satisfiable w.r.t.\ $T(\Omc)$.
In general, nonetheless, we point out that \emph{any} DL satisfying the two
conditions stated in Theorem~\ref{prop:rldhlass} can be chosen for our
purposes, and that the same theoretical guarantees can be obtained over the
SMT-based backward reachability procedure.
As future work, we thus intend to introduce a more general framework for
DL-based SASs that is able to account for different DLs.
We also intend to extend the results obtained here to more sophisticated
artifact-centric models, such as the \emph{relational artifact systems}
(\emph{RASs}) studied in~\cite{CalEtAl19,CalEtAl20}.
Moreover, it could be worth investigating in this setting also properties that go beyond safety, such as liveness and fairness.

\section*{Acknowledgements}

This research has been partially supported
by the Wallenberg AI, Autonomous Systems and Software Program (WASP) funded by
the Knut and Alice Wallenberg Foundation,
by the Italian Basic Research (PRIN) project HOPE,
%
by the EU H2020 project INODE
(grant agreement 863410)
by the CHIST-ERA project PACMEL,
%
by the project IDEE (FESR1133) funded by the Eur.\ Reg.\ Development
Fund (ERDF) Investment for Growth and Jobs Programme 2014-2020,
and by the Free University of Bozen-Bolzano through the projects
KGID, 
GeoVKG, 
 STyLoLa,
 VERBA and MENS.


\clearpage

\bibliographystyle{splncs04}
\bibliography{local}

\clearpage

\appendix

\section{Appendix}
\label{sec:app}


\subsection*{Proofs for Section~\ref{sec:firstres}}
\label{sec:fomod}

%
%
%
In the following, given a signature $\Sigma$, we call a \emph{$\Sigma$-interpretation} an interpretation $\Imc = (\Delta^{\Imc}, \cdot^{\Imc})$, where the domain of $\cdot^{\Imc}$ is restricted to $\Sigma$.
Let $\Imc$ and $\Jmc$ be $\Sigma$-interpretations. A \emph{$\Sigma$-homomorphism} (or simply \emph{homomorphism}) from $\Imc$ to $\Jmc$ is a function $\mu \colon \Delta^{\Imc} \longrightarrow \Delta^{\Jmc}$, denoted by $\mu \colon \Imc \longrightarrow \Jmc$, satisfying the following conditions, for every $\underline{d}$ in $\Delta^{\Imc}$:
\begin{enumerate}
	\item[$(1)$] for every individual symbol $a \in \Sigma$, $\mu(a^{\Imc}) = a^{\Jmc}$;
	\item[$(2)$] for every function symbol $f \in \Sigma$, $\mu(f^{\Imc}(\underline{d})) = f^{\Jmc}(\mu(\underline{d}))$;
	\item[$(3)$] for every predicate symbol $P \in \Sigma$, if $\underline{d} \in P^{\Imc}$, then $\mu(\underline{d}) \in P^{\Jmc}$.
\end{enumerate}
We say that a
homomorphism $\mu \colon \Imc \longrightarrow \Jmc$ is an
\emph{embedding} from $\Imc$ to $\Jmc$
if $\mu$ is injective
and such that:
\begin{enumerate}
	\item[$(3')$] for every predicate symbol $P \in \Sigma$, $\underline{d} \in P^{\Imc}$ iff $\mu(\underline{d}) \in P^{\Jmc}$.
\end{enumerate}

We say that $\Imc$ is a \emph{substructure} of $\Jmc$, and that $\Jmc$ is an \emph{extension} of $\Imc$, written $\Imc \subseteq \Jmc$, iff $\Delta^{\Imc} \subseteq \Delta^{\Jmc}$ and the
identity inclusion
$i \colon \Delta^{\Imc} \longrightarrow \Delta^{\Jmc}$ is an embedding from $\Imc$ to $\Jmc$.

A theory $T$ has the \emph{amalgamation property} if, for every pair of embeddings $\mu_{1} \colon \Imc_{0} \longrightarrow  \Imc_{1}$, $\mu_{2} \colon \Imc_{0} \longrightarrow  \Imc_{2}$ between models $\Imc_{0}$ and $\Imc_{1}, \Imc_{2}$ of $T$, there exist a model $\Imc$ of $T$ and embeddings $\nu_{1} \colon \Imc_{1} \longrightarrow \Imc$, $\nu_{2} \colon \Imc_{2} \longrightarrow \Imc$, such that $\nu_{1} \circ \mu_{1} = \nu_{2} \circ \mu_{2}$. The triple $(\Imc, \nu_{1}, \nu_{2})$ (or, with an abuse of notation, just $\Imc$) is called a \emph{$T$-amalgam} of $\Imc_{1}, \Imc_{2}$ over $\Imc_{0}$.

\rldhlass*
\begin{proof}
%
Concerning Property~$(i)$,
we have that the constraint satisfiability problem for $\sttr{\Omc}$ can be reduced to the
$\ourDL$
ontology satisfiability problem, seen as a fragment of $\mathcal{ALCHI}$, for which this problem is known to be decidable~\cite{Rud11}.
Indeed, let $\bigwedge_{i = 1}^{n} \alpha_{i}(\underline{x})$ be a conjunction of $\sig{\sttr{\Omc}}$-literals.
We have that $\bigwedge_{i = 1}^{n} \alpha_{i}(\underline{x})$ is satisfiable w.r.t. $\sttr{\Omc}$ iff 
$\sttr{\Omc} \cup \bigcup_{i = 1}^{n} \{ \alpha_{i}(\underline{x}) \}$ is satisfiable.
	Since each $\alpha_{i}(\underline{x})$ is a $\sig{\sttr{\Omc}}$-literal, the previous set of formulas is in turn satisfiable 
	iff $\sttr{\Omc} \cup \bigcup_{i = 1}^{n}\{ \alpha_{i}(\underline{a}) \}$ is satisfiable,
	with fresh $\underline{a} \in \NI \setminus \sig{\sttr{\Omc}}$.
	By definition, $\alpha_{i}(\underline{a})$ is the image under the standard translation of an assertion, and thus $\sttr{\Omc} \cup \bigcup_{i = 1}^{n} \{ \alpha_{i}(\underline{a}) \}$ is the image under the standard translation of an
	$\mathcal{ALCHI}$
	ontology, for which satisfiability can be decided.
	
To show Property~$(ii)$, we first require the following claim.

\begin{claim}
Given a signature $\Sigma$ without function symbols,
every universal $\Sigma$-theory $T$ that has the amalgamation property admits a model completion.
\end{claim}
\begin{proof}
Cf.~\cite{wheeler,LIP,CalEtAl20}.
\end{proof}

By the previous claim, since $\sig{T(\Omc)}$ does not contain function symbols and $T(\Omc)$ is universal, to prove that $T(\Omc)$ admits a model completion, it is enough to show that $T(\Omc)$ has the amalgamation property.
We prove this as follows.
Consider $\sig{T(\Omc)}$-interpretations $\Imc_{1}$ and $\Imc_{2}$ that are models of $T(\Omc)$, and let $\Imc_0$ be
a substructure of both $\Imc_{1}$ and $\Imc_{2}$ that is a model of $T(\Omc)$ (we assume w.l.o.g. that $\Delta^{\Imc_{1}} \cap  \Delta^{\Imc_{2}} = \Delta^{\Imc_{0}})$.
We define the $\sig{T(\Omc)}$-interpretation $\Imc = (\Delta^{\Imc}, \cdot^{\Imc})$ as follows:
\begin{itemize}
	\item $\Delta^{\Imc} = \Delta^{\Imc_{1}} \cup \Delta^{\Imc_{2}}$;
	\item for every individual symbol $a \in \sig{T(\Omc)}$, $a^{\Imc} = a^{\Imc_{1}}$;
	\item for every ($1$- or $2$-ary) predicate symbol $P \in \sig{T(\Omc)}$, $P^{\Imc} = P^{\Imc_{1}} \cup P^{\Imc_{2}}$.
\end{itemize}
Observe that $a^{\Imc} = a^{\Imc_{1}} = a^{\Imc_{0}} = a^{\Imc_{2}}$. Moreover, if $\underline{d} \in P^{\Imc}$, where $P$ is $n$-ary, for $n \in \{ 1, 2 \}$, then $\underline{d} \in (\Delta^{\Imc_{1}})^{n}$ and $\underline{d} \in P^{\Imc_{1}}$, or $\underline{d} \in (\Delta^{\Imc_{2}})^{n}$ and $\underline{d} \in P^{\Imc_{2}}$: this follows from the definition of $P^{\Imc}: =  P^{\Imc_{1}} \cup P^{\Imc_{2}}$.
Clearly, given embeddings $\mu_{1} \colon \Imc_{0} \longrightarrow  \Imc_{1}$, $\mu_{2} \colon \Imc_{0} \longrightarrow  \Imc_{2}$, the (inclusion) embeddings $i_{1} \colon \Imc_{1} \longrightarrow \Imc$, $i_{2} \colon \Imc_{2} \longrightarrow \Imc$ are such that $i_{1} \circ \mu_{1} = i_{2} \circ \mu_{2}$.
Thus, to show that $(\Imc, i_{1}, i_{2})$ is a $T(\Omc)$-amalgam of $\Imc_{1}$, $\Imc_{2}$ over $\Imc_{0}$, we have to prove that $\Imc$ is a model of $T(\Omc)$.
A formula $\p$ of $T(\Omc)$ has one of the following forms (we recall that formulas of $T(\Omc)$ are
given
in prenex normal form):
\begin{enumerate}
[label=$(\arabic*)$, align=left, leftmargin=*]
	\item $\forall x (A_{1}(x) \land \ldots \land A_{n}(x) \to \lambda(x))$;
	\item $\forall x \forall y (R_{1}(x,y) \to \lambda(x))$;
	\item $\forall x \forall y (R_{1}(x,y) \land A(y) \to \lambda(x))$;
	\item $\forall x \forall y (R_{1}(x,y) \to R_{2}(x,y))$;
	\item $\forall x \forall y (R_{1}(x,y) \to \lnot R_{2}(x,y))$;
\end{enumerate}
where: $A_{k} \in \NC$, for $k \in \{ 1, \ldots, n \}$; $\lambda \in \{ B, \lnot B \}$, with $B \in \NC$; $R_{i}(x,y) = P_{i}(x,y)$, if $R_{i} = P_{i}$, and $R_{i}(x,y) = P_{i}(y,x)$, if $R_{i} = P_{i}^{-}$, with $P_{i} \in \NR$ and $i \in \{ 1, 2 \}$.
We now show, reasoning by cases, that for every $j \in \{ 1, \ldots, 5 \}$ and every formula $\p \in T(\Omc)$
of the form $(j)$, $\Imc$ is a model of $\p$.

\begin{enumerate}
[label=$(\arabic*)$, align=left, leftmargin=*]
	\item Given $d \in \Delta^{\Imc}$, suppose that $\Imc \models A_{k}[d]$, i.e., $d \in A_{k}^{\Imc}$, for all $k \in \{ 1, \ldots, n \}$. By what already observed, we have that $d \in \Delta^{\Imc_{i}}$ and $d \in A_{k}^{\Imc_{i}}$, for $i = 1$ or $i = 2$, and thus $\Imc_{i} \models A_{k}[d]$. Since $\Imc_{i}$ is a model of $T(\Omc)$, and hence of $\p$, we have $\Imc_{i} \models \lambda[d]$. Given that $\lambda(x)$ is a literal and $\Imc_{i}$ is embedded in $\Imc$, we obtain that $\Imc \models \lambda[d]$, and thus $\Imc \models \p$.
	\item 
	Let $R_{1}(x,y) = P(x,y)$. Given a pair $(d,e)$ with $d,e \in \Delta^{\Imc}$, suppose that $\Imc \models R_{1}[d,e]$, 
	i.e., $(d,e) \in R_{1}^{\Imc}$. By what already observed, we have that $d,e \in \Delta^{\Imc_{i}}$, $(d,e) \in R_{1}^{\Imc_{i}}$, for $i = 1$ or $i = 2$, and thus $\Imc_{i} \models R_{1}[d,e]$. Since $\Imc_{i}$ is a model of $T(\Omc)$, and hence of $\p$, we have $\Imc_{i} \models \lambda[d]$. Given that $\lambda(x)$ is a literal and $\Imc_{i}$ is embedded in $\Imc$, we obtain that $\Imc \models \lambda[d]$, and thus $\Imc \models \p$. The case of $R_{1} = P(y,x)$ is analogous.
	\item Let $R_{1}(x,y) = P(x,y)$. Given a pair $(d,e)$ with $d,e \in \Delta^{\Imc}$, suppose that $\Imc \models R_{1}[d,e]$ and $\Imc \models A[e]$, i.e., $(d,e) \in R_{1}^{\Imc}$ and $e \in A^{\Imc}$. By what already observed, we have that $d,e \in \Delta^{\Imc_{i}}$, $(d,e) \in R_{1}^{\Imc_{i}}$, and $e \in A^{\Imc_{i}}$, for $i = 1$ or $i = 2$, and thus $\Imc_{i} \models R_{1}[d,e]$ and $\Imc_{i} \models A[e]$. Since $\Imc_{i}$ is a model of $T(\Omc)$, and hence of $\p$, we have $\Imc_{i} \models \lambda[d]$. Given that $\lambda(x)$ is a literal and $\Imc_{i}$ is embedded in $\Imc$, we obtain that $\Imc \models \lambda[d]$, and thus $\Imc \models \p$. The case of $R_{1} = P(y,x)$ is analogous.
	\item Let $R_{i}(x,y) = P_{i}(x,y)$, for $i \in \{ 1, 2 \}$. Given a pair $(d,e)$ with $d,e \in \Delta^{\Imc}$, suppose that $\Imc \models R_{1}[d,e]$, i.e., $(d,e) \in R_{1}^{\Imc}$. By what already observed, we have that $d,e \in \Delta^{\Imc_{i}}$ and $(d,e) \in R_{1}^{\Imc_{i}}$, for $i = 1$ or $i = 2$, and thus $\Imc_{i} \models R_{1}[d,e]$. Since $\Imc_{i}$ is a model of $T(\Omc)$, and hence of $\p$, we have $\Imc_{i} \models R_{2}[d,e]$. Given that $R_{2}(x,y)$ is a literal and $\Imc_{i}$ is embedded in $\Imc$, we obtain that $\Imc \models \lambda[d]$, and thus $\Imc \models \p$. The cases with $R_{i} = P_{i}(y,x)$, for $i = 1$ or $i = 2$, are analogous.
	\item This case is analogous to the previous one, since $\lnot R_{2}(x,y)$ is a literal and $\Imc_{i}$ is embedded in $\Imc$.
\end{enumerate}

Thus, we conclude that $\Imc \models \p$, for every $\p \in T(\Omc)$, i.e., $\Imc \models T(\Omc)$.
This completes the proof that $\Imc$ is a $T(\Omc)$-amalgam of $\Imc_{1}$, $\Imc_{2}$ over $\Imc_{0}$, hence $T(\Omc)$ has the amalgamation property.
By the claim above, we obtain that $T(\Omc)$ admits a model completion.

	Following an analogous argument to the one used in~\cite[Proposition 3.2]{CalEtAl20}, one can also exhibit the algorithm for quantifier elimination in the model completion of $\sttr{\Omc}$: given a formula $\exists x \p(x,\uy)$, take the quantifier-free formula $\psi(\uy)$ as the conjunction of the set of all quantifier-free formulae $\chi(\uy)$ such that $\p(x,\uy)\rightarrow \chi(\uy)$ is a logical consequences of $\sttr{\Omc}$ (they are clearly finitely many, up to $\sttr{\Omc}$-equivalence).

\end{proof}




\subsection*{Proofs for Section~\ref{sec:dlsas}}

\begin{lemma}
\label{lem:nocasedeffunctsas}
The safety problem for an \ourDL-SAS $\Smc$ can be 
reduced to the safety problem for an \ourDL-SAS $\Smc'$ (the size of which is polynomial in the size of $\Smc$) without any occurrence of case-defined functions.
\end{lemma}
\begin{proof}
We first require the following claim.
\begin{claim}
For every \ourDL-SAS
$\Smc =  (\Omc, \underline{x}, \iota(\underline{x}), \bigcup_{j = 1}^{m} \{ \tau_{j} (\underline{x}, \underline{x}') \})$, there exists an \ourDL-SAS $\Smc' =  (\Omc, \underline{x}, \iota(\underline{x}), \bigcup_{j = 1}^{m'} \{ \tau'_{j} (\underline{x}, \underline{x}') \})$ such that:
$(i)$~the transition formulas $\tau'_{j}$, for $0 \leq j \leq m'$, do not contain any case-defined function;
and
$(ii)$ for every safety formula $\nu(\underline{x})$, there exists a formula $\vartheta^{k}_{\Smc,\nu}$ of the form (\ref{eq:safetyform}) that is satisfiable w.r.t. $\Omc$, for some $k\geq 0$, iff there exists a formula $\vartheta^{k'}_{\Smc',\nu}$ of the form
 (\ref{eq:safetyform}) 
that is satisfiable w.r.t. $\Omc$, for some $k'\geq 0$.
The construction of $\Smc'$ is polynomially long in the size of $\Smc$ (i.e., $O(n^2)$, when $n$ is the overall size of $\Smc$).
\end{claim}
\begin{proof}[Proof of Claim]
%
First, concerning Point~$(i)$, 
we observe that, given an \ourDL ontology $\Omc$ and a formula $\p$ containing a case-defined function $F$ (based on some $\Omc$-partition and list of $\sig{\Omc}$-terms), there exists a formula $\p^{\dagger}$ not containing any occurrence of $F$ and such that $\Omc \models \p \leftrightarrow \p^{\dagger}$.
Such a formula $\p^{\dagger}$ can be obtained from $\p$ by substituting every atom $\alpha$ in which $F(\underline{x})$ occurs by the disjunction $\bigvee_{i = 1}^{n} (\kappa_{i}(\underline{x}) \land \alpha(t_{i}(\underline{x}))$.
Now, 
consider $\Smc =  (\Omc, \underline{x}, \iota(\underline{x}), \bigcup_{j = 1}^{m} \{ \tau_{j} (\underline{x}, \underline{x}') \})$
and let $1 \leq h \leq m$ be such that
$\tau_{h} (\underline{x}, \underline{x}') = \exists \underline{y} (\gamma^{h}(\underline{x}, \underline{y}) \land \bigwedge_{i = 1}^{n} x'_{i}  = F^{h}_{i}(\underline{x}, \underline{y}))$, where $F^{k}_{i}$ are cased-defined functions based on some $\Omc$-partitions and lists of $\sig{\Omc}$-terms.
By substituting 
each
$x'_{i}  = F^{h}_{i}(\underline{x}, \underline{y})$
with
$\bigvee_{j = 1}^{p} ( \kappa_{i}^{h,j}(\underline{x}) \land x'_{i}  = t_{i}^{h,j}(\underline{x}, \underline{y}) )$
and applying first-order logic transformations,
we obtain the formula
$\bigvee_{j = 1}^{p} \tau^{\dagger}_{h, j} (\underline{x}, \underline{x}')$,
where
\[
	\tau^{\dagger}_{h, j} (\underline{x}, \underline{x}') = \exists \underline{y}(\bigwedge_{i = 1}^{n} (\gamma^{h}(\underline{x}, \underline{y}) \land \kappa_{i}^{h,j}(\underline{x}) ) \land \bigwedge_{i = 1}^{n} x'_{i}  = t_{i}^{h,j}(\underline{x}, \underline{y}) ),
\]
%
which is equivalent to the original $\tau_{h} (\underline{x}, \underline{x}')$ in all models of $\Omc$. Notice that the size of each $\tau^{\dagger}_{h, j}$ is clearly $O(n)$, where $n$ is the overall size of the input $\Smc$, 
 and, since $p=O(n)$, there are $O(n^2)$ such $\tau^{\dagger}_{h, j}$. Hence, the size of $\Smc'$ is  $O(n^2)$, as wanted.
 By
taking,
for all the relevant $1 \leq h \leq m$,
\[
\bigl(\bigcup_{j = 1}^{m} \{ \tau_{j} (\underline{x}, \underline{x}') \} \setminus \{ \tau_{h} (\underline{x}, \underline{x}') \} \bigr)
\cup
\bigcup_{j = 1}^{p}  \{ \tau^{\dagger}_{h, j} (\underline{x}, \underline{x}') \},
\]
we obtain an \ourDL-SAS $\Smc' = (\Omc, \underline{x}, \iota(\underline{x}), \bigcup_{j = 1}^{m'} \{ \tau'_{j} (\underline{x}, \underline{x}') \})$
that does not contain any case defined function.
Moreover, concerning Point~$(ii)$, it is straightforward to see that, since $\Omc \models  \bigvee_{j = 1}^{m}  \tau_{j} \leftrightarrow \bigvee_{j = 1}^{m'}  \tau'_{j}$, 
for every safety formula $\nu(\underline{x})$, 
there exists a formula $\vartheta^{k}_{\Smc, \nu}$ of the form
(\ref{eq:safetyform}) that is satisfiable w.r.t. $\Omc$, for some $k \geq 0$, iff there exists a formula $\vartheta^{k'}_{\Smc', \nu}$ of the form
(\ref{eq:safetyform})
that is satisfiable w.r.t. $\Omc$, for some $k'\geq 0$.


\end{proof}

The statement of the lemma then follows from the previous claim: it is sufficient to notice that,
for every safety formula $\nu(\underline{x})$,
there is no $k\geq 0$ and no formula $\vartheta^{k}_{\Smc, \nu}$ of the form
(\ref{eq:safetyform}) that is satisfiable w.r.t. $\Omc$ iff there is no $k'\geq 0$ and no formula
$\vartheta^{k'}_{\Smc', \nu}$ of the form (\ref{eq:safetyform}) 
that is satisfiable w.r.t. $\Omc$. This implies that  $\Smc$ is safe w.r.t. $\nu(\underline{x})$ iff $\Smc'$ is so, where $\Smc'$ does not contain case-defined functions.

\end{proof}

\soundcomplete*
\begin{proof}
Thanks to Lemma~\ref{lem:nocasedeffunctsas}, in the rest of this proof we assume w.l.o.g. that an \ourDL-SAS $\Smc = (\Omc, \underline{x}, \iota(\underline{x}), \bigcup_{j = 1}^{m} \{ \tau_{j} (\underline{x}, \underline{x}') \})$ does not contain case-defined functions.
Moreover,
we require the following claim, an immediate consequence of the definitions.

\medskip
\noindent
\textit{Claim~\ref{cla:tracesat}.}
For every safety formula $\nu(\underline{x})$ for $\Smc$ and every $k \geq 0$, a formula $\vartheta$ of the form~\eqref{eq:safetyform} is satisfiable w.r.t. $\Omc$  iff $\vartheta$ is satisfiable w.r.t. $T(\Omc)$.

\medskip

Now,
	we show that, instead of considering satisfiability of formulae of the form~\eqref{eq:safetyform} in models of $T(\Omc)$, we can concentrate on $T(\Omc)^{\ast}$-satisfiability.
	 	By definition,
	an \ourDL-SAS $\Smc$ is unsafe iff there is a formula $\vartheta$ of the form~\eqref{eq:safetyform}	\[
	\iota(\underline{x}^{0}) \land \tau_{j_{0}}(\underline{x}^{0}, \underline{x}^{1}) \land \ldots \land \tau_{j_{k-1}}(\underline{x}^{k-1}, \underline{x}^{k}) \land \nu(\underline{x}^{k})
	\]
	 that is satisfiable w.r.t. $\Omc$, for some $k \geq 0$. 
	By
	the observation
	above,
	we get that $\Smc$ is unsafe iff  there is a formula $\vartheta$ of the form~\eqref{eq:safetyform} that is satisfiable
w.r.t.
	$T(\Omc)$, for some $k \geq 0$.
		Thanks to Theorem~\ref{prop:rldhlass},
		$T(\Omc)$ admits a model completion $T(\Omc)^{\ast}$. Hence, 
	since the formulas of the form~\eqref{eq:safetyform} are existential
	$\sig{T(\Omc)}$-formulas, and by using the property that every model of an FO theory $T$  embeds into a model of its model completion $T^*$, we conclude that $\Smc$ is unsafe iff, for some $k\geq 0$, there is a formula  $\vartheta$ of the form~\eqref{eq:safetyform} that is satisfiable in a model of $T(\Omc)^{\ast}$.
	Thus, for establishing (un)safety of $\Smc$, we can concentrate on satisfiability of formulas of the form~\eqref{eq:safetyform} in models of $T(\Omc)^{\ast}$.

	We now continue the proof by adopting arguments similar to the ones contained in the proof of~\cite[Theorem 4.2]{CalEtAl20}.
	We want to show the correctness of Algorithm~\ref{fig:algorithm}. 	
	First, we preliminarily give some useful remarks on the algorithm.
	Let us call $B_n$ (resp. $\phi_n$), with $n\geq 0$, the status of the variable $B$ (resp. $\phi$) after $n$ executions in Line 4 (resp. Line 6) of Algorithm~\ref{fig:algorithm} ($n=0$ corresponds to the status of the variables in Line~1). 
	Notice that we have
	\begin{equation}\label{eq:pre-phi}
	T(\Omc)^*\models \phi_{j+1}\leftrightarrow Pre(\tau, \phi_j)
	\end{equation} for all $j$ and that
	\begin{equation}\label{eq:invariant}
	T(\Omc)\models B_n \leftrightarrow \bigvee_{0\leq j<n} \phi_j
	\end{equation}
	is an invariant of the algorithm.

	We now show that if the algorithm returns an $\mathsf{unsafe}$ outcome, this outcome is correct, i.e., $\Smc$ is really unsafe.
	Since we are considering satisfiability of formulae of the form~\eqref{eq:safetyform} in models of $T(\Omc)^{\ast}$, we can apply the quantifier elimination procedure of $T(\Omc)^{\ast}$: it can be easily seen that the satisfiability of the quantifier-free formula we get in this way is equivalent to the satisfiability of
	$\iota\wedge \phi_n$: clearly, this is again a quantifier-free formula (because of line 6 of Algorithm~\ref{fig:algorithm}). Since $T(\Omc)$-satisfiability and $T(\Omc)^{\ast}$-satisfiability are equivalent (by definition of model completion) when dealing with existential (and in particular, quantifier-free) formulae, the $T(\Omc)$-satisfiability of $\iota\wedge \phi_n$	%
	is decidable thanks to Theorem~\ref{prop:rldhlass}.
	Hence, if Algorithm~\ref{fig:algorithm} terminates with an
	$\mathsf{unsafe}$ outcome, then there exists a formula $\vartheta$ of the form~\eqref{eq:safetyform}  that is $T(\Omc)^{\ast}$-satisfiable.  This exactly means that $\Smc$ is unsafe, as wanted.
	
	We now show that if the algorithm returns a $\mathsf{safe}$ outcome, this outcome is correct, i.e., $\Smc$ is really safe.
	Consider the satisfiability test in Line~2. This is again a satisfiability test for a quantifier-free $\sig{T(\Omc)}$-formula, thus it is decidable. In case of a $\mathsf{safe}$ outcome, we have that $T(\Omc) \models \phi_n\to B_n$; 
	we claim that, if we continued executing the loop of Algorithm~\ref{fig:algorithm}, we would get that
\begin{equation}\label{eq:claim-safe}
T(\Omc)^{\ast}\models B_m\leftrightarrow B_n,
\end{equation}
for all $m\geq n$.
We justify Claim~\eqref{eq:claim-safe} below. 

From $T(\Omc) \models \phi_n\to B_n$, taking into  consideration that $T(\Omc)^{\ast}\supseteq T(\Omc)$ 
and that Formula~\eqref{eq:pre-phi} 
holds,
we get $T(\Omc)^{\ast}\models\phi_{n+1} \to Pre(\tau,B_n)$. Since $Pre$ commutes with disjunctions (i.e., $Pre(\tau,\bigvee_j \phi_j)$
is logically equivalent to $\bigvee_j Pre(\tau,\phi_j)$),  we also have $T(\Omc)^{\ast}\models Pre(\tau,B_n)\leftrightarrow \bigvee_{1\leq j\leq n} \phi_j$ by the Invariant~\eqref{eq:invariant} and
by Formula~\eqref{eq:pre-phi}
again.  
By using the entailment $T(\Omc) \models \phi_n\to B_n$ once more,  we get  $T(\Omc)^{\ast}\models\phi_{n+1} \to B_n$ and also that $T(\Omc)^{\ast}\models B_{n+1}\leftrightarrow B_n$, thus we finally obtain that $T(\Omc)^{\ast}\models \phi_{n+1} \to B_{n+1}$. Since $\phi_{n+1} \to B_{n+1}$ is quantifier-free, $T(\Omc)^{\ast}\models \phi_{n+1} \to B_{n+1}$ implies $T(\Omc) \models \phi_{n+1} \to B_{n+1}$. This argument can be repeated for all $m\geq n$, obtaining that $T(\Omc)^{\ast}\models B_m\leftrightarrow B_n$ for all $m\geq n$, i.e., Claim~\eqref{eq:claim-safe}.

	This would entail that $\iota\wedge \phi_m$ is always unsatisfiable (because of~\eqref{eq:invariant} and because $\iota\wedge \phi_j$ was unsatisfiable
	for all $j<n$), which is the same (as remarked above) as saying that all formulae~\eqref{eq:safetyform} are $T(\Omc)^{\ast}$-unsatisfiable. Thus, $\Smc$ is safe.
	
	\end{proof}

	\dlsasdec*
	\begin{proof}
	
	There are only finitely many 
	quantifier-free $\Sigma_{T(\Omc)}$-formulas, up to $T(\Omc)$-equivalence, that could be built out of a finite set of variables $\ux$: this holds for every \ourDL ontology $\Omc$.
	Thanks to the quantifier elimination procedure in Line 6, the overall number of variables in $\phi$ is never increased: notice that, without quantifier elimination, computing preimages $\Pre(\tau, \phi_j)$ would introduce in $\phi_{j+1}$ new quantified variables, because of the presence of existentially quantified variables $\uy$ in $\tau$.  This implies that globally there are only finitely many quantifier-free $\Sigma_{T(\Omc)}$-formulas that Algorithm~\ref{fig:algorithm} needs to analyse. Hence, Algorithm~\ref{fig:algorithm} must terminate:
	because of~\eqref{eq:invariant},
	the unsatisfiability test of Line 2 must eventually succeed, if the unsatisfiability test of Line 3 never does so.

	Concerning complexity, we need to modify Algorithm~\ref{fig:algorithm}. We first notice that, thanks to Lemma~\ref{lem:nocasedeffunctsas}, the preprocessing
	that converts an \ourDL-SAS with occurrences of case-defined functions into its equivalent \ourDL-SAS without any occurrence of case-defined functions does not increase the overall complexity of the problem. Moreover, the translation of an \ourDL-ontology $\Omc$ into $T(\Omc)$ requires polynomial time. Then, we adopt a nondeterministic procedure, analogous to the one in ~\cite[Theorem 6.1]{CalEtAl19}, that makes the complexity $\NPSpace$: the
	main
	difference from~\cite[Theorem 6.1]{CalEtAl19} is that in our signatures, instead of unary functions, we have unary and binary relational symbols, but the argument works similarly. By Savitch's Theorem ($\PSpace = \NPSpace$), we conclude the proof.

\end{proof}


\end{document}